\documentclass[letterpaper]{article}
\usepackage{uai2019}
\usepackage[margin=1in]{geometry}

\usepackage{soul}
\usepackage{url}
\usepackage[hidelinks]{hyperref}
\usepackage[utf8]{inputenc}
\usepackage[small]{caption}
\usepackage{graphicx}

\usepackage{natbib}

\usepackage{amsthm}
\usepackage{amssymb}
\usepackage{amsmath}
\usepackage{cases}

\usepackage{booktabs}
\usepackage{algorithm}
\usepackage{algorithmic}

\usepackage{centernot}

\usepackage{comment}

\theoremstyle{definition}
\newtheorem{definition}{Definition}

\newtheorem{theorem}{Theorem}
\newtheorem{lemma}{Lemma}
\newtheorem{proposition}{Proposition}

\newtheorem*{priorwork*}{Prior Work}

\newcommand{\nrightsquigarrow}{\centernot{\rightsquigarrow}}

\newcommand{\T}{\mathcal{T}}
\newcommand{\Ttotal}{\mathcal{T}}

\newcommand{\M}{\mathcal{M}}
\newcommand{\Mtotal}{\mathcal{M}}
\newcommand{\Mlocal}{\mathcal{M}_{\mathrm{local}}}

\newcommand{\aT}{\mathsf{T}}

\renewcommand{\L}{\mathcal{L}}

\newcommand{\totalpartial}[2]{{#1}}

\title{On Open-Universe Causal Reasoning \thanks{\; In \emph{Proceedings of the 35th Conference on Uncertainty in Artificial Intelligence} (UAI 2019).}}

\author{ {\bf Duligur Ibeling} \\
Computer Science Dept. \\
Stanford University\\
Stanford, CA 94305 \\
\And
{\bf Thomas Icard} \\
Philosophy Dept. \\
Stanford University \\
Stanford, CA 94305 \\
}

\begin{document}

\maketitle

\begin{abstract}
We extend two kinds of causal models, structural equation models and simulation models, to infinite variable spaces. This enables a semantics for conditionals founded on a calculus of intervention, and axiomatization of causal reasoning for rich, expressive generative models---including those in which a causal representation exists only implicitly---in an open-universe setting.  Further, we show that under suitable restrictions the two kinds of models are equivalent, perhaps surprisingly as their axiomatizations differ substantially in the general case. We give a series of complete axiomatizations in which the open-universe nature of the setting is seen to be essential.
\end{abstract}

\section{Introduction} \label{intro}

A hallmark of intelligence is the ability to reason about cause and effect. Indeed, some form of causal representation is essential for robust reasoning under uncertainty, including flexible planning, delivering useful explanations, and the capacity to transfer knowledge from one domain to another. How should causal knowledge be represented? Extant formalisms vary widely both in terms of their representational primitives and in the reasoning principles they engender, witness frameworks based on similarity-orderings \citep{lewis73,Ginsberg}, Bayesian networks, and structural equation models \citep{Spirtes,Pearl2009}, among others. While distinct and sometimes even incompatible (see, e.g., \citealt{Halpern2013,Zhang} on similarity-orderings versus structural equation models, and \citealt{Pearl2009,Bottou} on Bayes nets versus structural equation models), each of these frameworks captures important insights about causal representation and reasoning.

A second hallmark of uncertain reasoning is the ability to deal with ``open-universe'' or ``first-order'' domains in a way that does not assume a fixed, finite set of variables. \cite{BLOG} present an illustrative example of aircrafts producing identical blips on a radar screen. The radar blips give a noisy picture of the positions and velocities of some antecedently unknown number of aircrafts. Reasoning about evident in this setting introduces distinct challenges, chiefly to encode evidence about objects one did not know existed. Observing $N$ blips, for example, is still consistent with any number of latent aircrafts. The field of statistical relational modeling has developed sophisticated methods for learning and reasoning in such domains \citep{Poole,BLOG,Carbonetto,Markov,Srivastava,Gogate}.

Research on these two aspects of uncertain reasoning has proceeded mostly independently. Whereas similarity-based causal representations in principle extend to first-order settings (e.g., \citealt{FriedmanKollerHalpern}), and there have been numerous extensions of directed graphical models to handle unbounded numbers of variables (e.g., \citealt{Friedman,Pfeffer}), models with a perspicuous and useful causal interpretation typically assume a fixed, finite set of variables. That is, although in the course of learning one may always infer the existence of new hidden variables, no inferred causal model would typically employ an unbounded set of variables. Theoretical work on causal representation nearly always assumes this restriction \citep{GallesPearl,Halpern2000,Spirtes,Zhang,Pearl2009}.

Such a restriction is undoubtedly appropriate for many applications. But it may also be limitative. Returning to the radar blip example, for instance, we should be able to assess such claims as, ``if there were more than 100 aircrafts, at least one would be missed by the radar,'' or ``given the observed blips, had there been five more aircrafts, at least two would have been dangerously close to one another,'' and so on. These are essentially first-order queries that depend on complex causal facts about an unknown number of entities and their properties and relations. Such causal queries ought to be answerable from a single underlying model (possibly under different observations), without having to specify a distinct set of variables for each use.

One promising approach to this problem comes from a tradition broadly within statistical relational modeling, whereby knowledge is encoded in the form of \emph{generative probabilistic programs}. Probabilistic programs can be used  as stochastic simulators and inverted to perform conditional inference given observational data, capturing an important variety of probabilisitic learning \citep{BLOG,goodmanetal08,deRaedt,tran2017deep,Pyro}. Encoding a generative process implicitly as a program facilitates succinct representation of complex dependencies among unbounded sets of variables.

Similarly, within research on neural networks, so called \emph{implicit generative models}---tools such as variational autoencoders and generative adversarial networks---also represent probability distributions implicitly by means of stochastic data simulators \citep{Mohamed}. Though these models are typically incapable of dealing with fully open-universe domains, extensions may in principle be expressive enough to do so (see, e.g., \citealt{Li:2017}). Related ideas involving ``mental simulation'' of environmental dynamics are being increasingly explored within neural approaches to reinforcement learning \citep{Hamrick}.

Probabilistic generative models are often associated with an intuitive causal interpretation, although this interpretation is not always appropriate \citep{Peters}. For example, there should be a reason for thinking that parameters of the model correspond appropriately to aspects of the true underlying data-generating process \citep{Besserve}. Our interest here is in a more abstract question: what structural features would be sufficient in order for an open-universe model even to be a candidate as a causal model? 

Our aim in this paper is to establish formal (axiomatic) foundations for genuinely \emph{open-universe causal models}. Specifically, we want to understand and assess the subjunctive conditional claims they encode. Subjunctive conditionals---``what if?'' statements about what would occur under counterfactual or hypothetical conditions---are of foundational importance for causal reasoning, arguably definitional of the subject matter \citep{Spirtes,Pearl2009,Bottou,Peters}. Such conditional statements are typically formalized by appeal to a notion of \emph{causal intervention} \citep{meekglymour94,Spirtes,Pearl2009}, and axioms for conditionals play crucially into prominent algorithms for inferring counterfactual and interventional probabilities (``causal effects'') from observational data \citep{Shpitser,Bareinboim,Hyttinen}.

Open-universe models can be specified either implicitly via a (data-generating) program or process, or explicitly via a set of equations relating the variables. In general these two kinds of models---together with natural corresponding concepts of hypothetical/counterfactual intervention---are quite different, both conceptually and axiomatically \citep{IbelingIcard}. For instance, they handle loops and feedback in incompatible ways (cf. also \citealt{Lauritzen}). Nonetheless, we show that under suitable restrictions, well-motivated from a causal perspective, the two are expressively equivalent. Thus, our first contribution is to define two natural classes of simulation programs and structural equation models, respectively, and show equivalence of these two classes with respect to the conditional claims they entail. The characterizations appeal to an implicit (discrete) temporal structure, and can be seen as ``open-universe data-generating processes'' generalizing the idea of a ``recursive'' causal model in the sense of \cite{Pearl2009}, viewed either procedurally or declaratively. We offer this as a formalization of those generative models that could plausibly be interpreted as causal models. 

Following this we establish a series of axiomatization results with respect to natural systems for reasoning about subjunctive conditionals. Our results build on previous axiomatic work on causal conditionals that assumed a fixed, finite set of variables \citep{GallesPearl,Halpern1998,Halpern2000}, extending this work to the open-universe setting. Dealing with an unbounded set of variables presents unique challenges. Nevertheless, we show that an axiomatic system including quintessential causal principles (including those used in aforementioned identifiability algorithms) is sound and complete for both the procedural and the declarative classes of models, thereby substantiating the causal interpretation of both classes. 
The satisfiability problem remains $\mathsf{NP}$-complete, showing that abstract causal reasoning (e.g., as features in the do-calculus) in the open-universe setting is no more complex than in the finite setting. 

Bringing out the open-universe aspect of these models, we consider augmenting the conditional language with a \emph{causal influence} relation $\rightsquigarrow$, whereby $X \rightsquigarrow Y$ expresses that there is a context in which a change in $X$ would lead to a change in $Y$ \citep{Woodward,Pearl2009}. Whereas this relation is easily definable in the finite setting \citep{Halpern1998}, in our setting $\rightsquigarrow$ is an essentially higher-order notion, quantifying over all possible interventions. \emph{Model checking} this relation in particular models---most notably in rich probabilistic programs---may in general be undecidable.

Finally, as an additional application of this study, we consider the special case where all causal influence is \emph{local}, in the sense that whenever $X \rightsquigarrow Y$, this influence is mediated by variables temporally in between $X$ and $Y$. As we explain, this semantic assumption is quite natural in the open-universe setting, and axiomatically it corresponds exactly to the claim that $\rightsquigarrow$ is a transitive relation.

The broader goal of this work is to establish foundations suitable for understanding and assessing highly expressive (open-universe) representational systems that encode causal structure, even if only implicitly. Such models feature not only in recent AI research, but also centrally in models of human cognition (see, e.g., \citealt{Freer,Lake2017}). Indeed, the open-universe nature of human reasoning under uncertainty is well-established \citep{Kemp}. We may thus expect that, however causal knowledge is represented in humans, it accommodates flexible, open-universe reasoning. At a high level, the present work brings together research on causal representation and reasoning with ideas and tools from statistical relational modeling and higher-order representation.



\section{Two Classes of Models}
In this section we introduce two kinds of models, one based on systems of equations (\S\ref{sec:sems}), the other based on algorithms (\S\ref{sec:simulations}). They respectively emphasize declarative and procedural styles of modeling. Owing in part to this difference, the most general versions of the two can be distinguished even at a very abstract, axiomatic level \citep{IbelingIcard}. However, by restricting to appropriate (causally motivated) subclasses, we can demonstrate their equivalence (Thm. \ref{equivalencetheorem}). 

Throughout this section we assume a signature $(\chi, \Sigma)$, specifying a countably infinite set $\chi$ of variables that take on values from a set $\Sigma$. For example, we might have a variable in $\chi$ representing ``the position of the 77th aircraft,'' which could take on values in some numerical range (or a special value ``undefined'' if there is no 77th aircraft). Even if we restrict each variable to take only finitely many possible values, the fact that $\chi$ is infinite allows encoding arbitrarily complex structures.

\subsection{Structural Equation Models} \label{sec:sems}
The most general class of structural equations models (see, e.g., \citealt{Pearl2009}) allows arbitrary equations among arbitrary sets of variables. Because our primary interest here is in (open-universe, probabilistic) generative models, we restrict attention to a special subclass of possible sets of equations. First, we assume an infinite set of variables $\chi$. Second, similar to so called recursive models \citep{Pearl2009}, we will assume that variables can be given a causal (intuitively, temporal) order. Although there may well be adequate causal interpretations of nonrecursive models, e.g., as descriptions of equilibrium behavior of some underlying process \citep{Strotz}, these interpretations rely on the declarative character of equational modeling and cannot always be adequately captured by an appropriate data-generating process \citep{Lauritzen}. 

Third, unlike in dynamic Bayesian networks and related extended graphical models \citep{Dean,Friedman}, we do not require each variable to depend on only finitely many others; a variable may depend on an unbounded number of variables preceding it in a temporal order. Fourth, we require all equations to be uniformly computable (Defn. \ref{computablesem}). 

Note, finally, that we make no use here of the distinction between exogenous and endogenous variables (though see the discussion in \S\ref{disc} below). 

\begin{definition}[Structural Equation Model] \label{semdefinition}
A \emph{structural equation model} (SEM) is a collection of 
partial functions $\{f_X\}_{X \in \chi}$, with $f_X : (\chi \to \Sigma) \to \Sigma$, and a \emph{time} map $t : \chi \to \mathbb{N}$. Each $f_X$ is a function only of preceding variables: if $v, v' \in \mathrm{dom}(f_X)$ and $v(X') = v'(X')$ for all $X'$ such that $t(X') < t(X)$, then
$f_X(v) = f_X(v')$.
\end{definition}
The functions in an SEM specify \emph{structural equations} $\{ X = f_X(\cdot)\}_{X \in \chi}$ that are to be simultaneously satisfied: an SEM $M$ has \emph{solution} $v: \chi \to \Sigma$ if $f_X(v) = v(X)$ for all $X$. We write $M \models v$ in this case. Since the order $t$ is well-founded and acyclic, $v$ may be built up iteratively, and is thus unique. Intervention on an SEM is defined standardly \citep{meekglymour94,Pearl2009}:
\begin{definition}[Intervention on an SEM]
\label{semintervention}
An intervention is a partial function $i : \chi \to \Sigma$. It specifies variables $\mathrm{dom}(i) \subseteq \chi$ to be held fixed, and the values to which they are fixed. Intervention $i$ induces a mapping of SEMs, also denoted $i$, as follows. Abbreviate $X \in \mathrm{dom}(i)$ as $X \in i$. Then $i(M)$ is identical to $M$, but with $f_X$ replaced by the constant function $f_X(\cdot) = i(X)$ for each $X \in i$. 
\end{definition}
Given $i, i'$ we say $i$ is a \emph{restriction} of $i'$ if $\mathrm{dom}(i)\subseteq \mathrm{dom}(i')$ and $i(X) = i'(X)$ for all $X \in i$. Note $i'$ may be a total function $v : \chi \to \Sigma$.
\totalpartial{We define a natural relation of \emph{causal influence} in SEMs as follows (after \citealt{Woodward,Pearl2009}).
\begin{definition}[Causal influence]
  \label{seminfluence}
  \totalpartial{Let $X, Y \in \chi$ with $X \neq Y$.
  Then $M \models X \rightsquigarrow Y$ (read \emph{$X$ influences $Y$}) if there is an intervention $i$ (the \emph{context}) and $x_1, x_2 \in \Sigma$ so that, for $j = 1, 2$ letting $i_{x_j}$ be the intervention fixing $X$ to $x_j$ and $v_j$ be the solution $i_{x_j}(i(M)) \models v_j$, we have $v_1(Y) \neq v_2(Y)$.}{Let $X, Y \in \chi$ with $X \neq Y$ and let $\mathsf{T}$ be a simulation. We say $X$ influences $Y$ in $\mathsf{T}$ and write $\mathsf{T} \models X \rightsquigarrow Y$ if there is an intervention $i$ and a value $x \in \Sigma$ such that the following holds. Let $i_{X \leftarrow x}$ be the intervention fixing $X$ to $x$. Then $i_{X \leftarrow x}(i(M)) \models Y \leftarrow y$ but $i(\mathsf{T}) \models Y \leftarrow y'$ for some $y \neq y'$. 
  }
  \end{definition}
}{Define causal influence in $\M$ as follows.
\begin{definition}[Causal influence]
$M \models X \rightsquigarrow Y$ for $X \neq Y$ if there is an intervention $i$ and $x \in \Sigma$ such that $(i_{X \leftarrow x} \circ i)(M)$ and $i(M)$ both have solutions, but the solutions differ at $Y$.
\end{definition}}
If $X \rightsquigarrow Y$ then $t(X) < t(Y)$; this is easy to see again from the fact that the temporal order implies that a solution may be assembled iteratively.
\totalpartial{}{The disanalogy with Defn. \ref{simulationinfluence} stems from the fact that although a simulation may write any number of variables, it always gives rise to a unique partial solution $v :\subseteq \chi \to \Sigma$, which does not hold in $\M$ as each structural function is evaluated at all of $v$. In $\Ttotal, \Mtotal$, the disanalogy vanishes.}

For comparison with simulation models, it is helpful to identify those SEMs whose (counterfactual) solutions are computable. Similar restrictions have been explored in the literature, e.g., on approaches to causal inference based on so called algorithmic independence \citep{IM,Peters}. 
This is the subclass in which the functions $\{f_X\}_{X \in \chi}$ are \emph{uniformly} computable and $t$ is also computable (both presuming an encoding of $\chi$):
\begin{definition}[Computable Structural Equation Model]
\label{computablesem}
Let $M = \{f_X\}_{X \in \chi}$ be an SEM with time map $t$, and let $F$ be the set of computable functions $\chi \to \Sigma$.
We say $M$ is \emph{computable} if $t$ is computable,
$\mathrm{dom}(f_X) = F$ for all $X$, and
$\phi : \chi \times F \to \Sigma$
defined by $\phi(X, v) = f_X(v)$ is a computable function.
\end{definition}
To clarify the sense in which $\phi$ is computable, let us associate each variable in $\chi$ with a square of an infinite Turing machine tape with alphabet $\Sigma$. We call such a tape a \emph{variable tape}. A variable tape (no square of which is blank) stores a function $\chi \to \Sigma$. That $\phi$ is computable means that it is computed by a machine $M$ with two input tapes, the second of which is a variable tape. When $M$ is run with the encoding of some $X \in \chi$ stored on the first tape and a computable function $v : \chi \to \Sigma$ stored on the second (variable) tape, $M$ halts outputting the value $\phi(X, v)$. $M$ need not halt given uncomputable input $v$.

Let us write $\mathcal{M}$ for the class of all computable SEMs; clearly $\mathcal{M}$ closes under computable intervention and every SEM in $\mathcal{M}$ has a unique solution.

Let $\Mlocal$ be the subclass of $\mathcal{M}$ in which every $X \in \chi$ only depends on immediately preceding variables: if $v, v' \in \mathrm{dom}(f_X)$ and $v(X') = v'(X')$ for all $X'$ such that $t(X') = t(X) - 1$, then $f_X(v) = f_X(v')$.
Causal influence is not transitive in SEMs generally, but it is transitive in $\Mlocal$ \citep{Pearl2009}. Furthermore, we can prove the following ``interpolation'' result:
\begin{proposition}
  \label{Mlocalmediated}
  Let $M \in \Mlocal$ and $X, Y \in \chi$. If $M \models X \rightsquigarrow Y$ and $t(Y) > t(X) + 1$, there is a variable $X'$ such that $M \models X \rightsquigarrow X'$ and $M \models X' \rightsquigarrow Y$.
\end{proposition}
\begin{proof}
There is a context $i$ and $x_j, v_j$ for $j = 1,2$ such that $i_{x_j}(i(M)) \models v_j$ and $v_1(Y) \neq v_2(Y)$ (Defn. \ref{seminfluence}); thus $\phi(Y, v_1) \neq \phi(Y, v_2)$ (Defn. \ref{computablesem}). Let $S_j \subset \chi$ be the set of variables whose squares the computation of $\phi(Y, v_j)$ accesses before outputting a value for $Y$ and halting; crucially $S_j$ is finite. Form intervention $i_j$ as the restriction of $v_j$ to $S_1 \cup S_2$ and let $u_j$ be the solution $i_j(M) \models u_j$. Since $M \in \Mlocal$, $u_j(Y) = v_j(Y)$. Let $\{X_i\}_{2 \le i \le k}$ be those variables in $\mathrm{dom}(i_1)$ for which $i_1(X_i) \neq i_2(X_i)$, and consider a finite sequence $\{i'_n\}_{1 \le n \le k}$ where $i'_1 = i_1$, and for $2 \le n \le k$, $i'_n = i'_{n-1}$ with the sole exception that $i'_n(X_n) = i_2(X_n)$. Let $\{v'_n\}_{1\le n \le k}$ be the solutions $i'_n(M) \models v'_n$. Since $i'_k = i_2$, we have $v'_k(Y) = u_2(Y) = v_2(Y) \neq v_1(Y) = u_1(Y) = v'_1(Y)$; there is thus some $m$, $2\le m \le k$ for which $v'_m(Y) \neq v'_{m-1}(Y)$. Then $i'_{m-1}$ witnesses that $M \models X'_{m} \rightsquigarrow Y$ and $i_1$ witnesses that $M \models X \rightsquigarrow X'_m$ as $v_1(X'_m) = i_1(X'_m) \neq i_2(X'_m) = v_2(X'_m)$.
\end{proof}
Taking the temporal interpretation of $t$ seriously, Prop. \ref{Mlocalmediated} reflects the intuition that causal influence is always mediated through time, essentially enforcing a kind of Markov assumption.

\subsection{Monotone Simulation Models} \label{sec:simulations}

Our second class of models similarly deals with an infinite class of variables, but is intended to capture the procedural character of a simulation program, e.g., as defined in an expressive programming language. Such models have been studied in relative generality by means of (probabilistic) Turing machines (see, e.g,. \citealt{Freer} for an overview), and at this level of generality they can be shown to validate only a very weak set of conditional axioms \citep{IbelingIcard}, far weaker than those commonly required of causal conditionals (cf. \S\ref{axn} below). Here we restrict attention to machines in which the variables can be given a causal order, again analogously to recursive graphical models. As simulation programs are intended to capture an underlying data-generating process, this causal order can be well justified (cf. \citealt{Lauritzen}).
\begin{definition}[Simulation model]
\label{simulation}
A \emph{simulation model} (or just a simulation) is a deterministic Turing machine with two tapes: a work tape and a variable tape (see discussion following Defn. \ref{computablesem}). The variable tape is write-only: no variable head transitions in which a symbol of $\Sigma$ is erased
or rewritten are allowed.
\end{definition}
\totalpartial{Thanks to the write-only restriction, any variable is written at most once. Given $v: \chi \to \Sigma$, if every $X$ is written with the value $v(X)$, $v$ is the solution of $\aT$ and we write $\aT \models v$.}{Write $\aT \models X \leftarrow x$ if $\aT$ writes to $X$ with value $x$. Thanks to the write-only restriction at most one such formula holds in $\aT$ for each $X$. Given $v : \chi \to \Sigma$, if $\aT \models X \leftarrow v(X)$ for all $X$, then we say $v$ is the \emph{solution} of $\aT$ and write $\aT \models v$.}
Intervention is defined via blocking rewrites \citep{AC17,IbelingIcard}:
\begin{definition}[Intervention]
\label{intervention}
Given intervention $i$ and an oracle for $i$, the simulation $i(\aT)$ emulates $\mathsf{T}$ but acts as if the square for any $X \in i$ is fixed to the value $i(X)$;\footnote{
  Formally, suppose at some point in the execution the next variable tape transition of $\mathsf{T}$ is $\delta(q, s) = (q', s', d)$.
  Before simulating this transition $i(\aT)$ calls the oracle on $X$.
  If $X \notin i$
  then the simulation proceeds unimpeded: $\mathsf{T}$ transitions to state
  $q'$, $s'$ is written to square $X$,
  and the head moves in direction $d$.
  If $X \in i$ then instead, the simulation does the action $(q'', i(X), d')$ where $\delta(q, i(X)) = (q'', s'', d')$.
  } it dovetails this emulation with a procedure that writes $i(X)$ to $X$ for all $X \in i$.
\end{definition}
In other words, an intervention on a simulation program $\mathsf{T}$ is an operation on the code of $\mathsf{T}$ that holds some set of variables to fixed values, an operation that will typically have side-effects on other variables. 

A computable intervention is one whose oracle can be effectively implemented; simulations thus close under computable intervention. The primary phenomena of interest are the behaviors of a simulation under possible interventions. This motivates an equivalence notion: 
\begin{definition}
\label{simequivalence}
\totalpartial{Simulations $\mathsf{T}, \mathsf{T}'$ are \emph{equivalent}, $\mathsf{T} \simeq \mathsf{T}'$, if for any computable $i$ and $v: \chi \to \Sigma$, $i(\mathsf{T}) \models v$ iff $i(\mathsf{T}') \models v$. They are \emph{weakly equivalent} if this property only holds when $\mathrm{dom}(i) = \varnothing$.}{Simulations $\mathsf{T}, \mathsf{T}'$ are \emph{equivalent}, $\mathsf{T} \simeq \mathsf{T}'$, if $i(\aT) \models X \leftarrow x \leftrightarrow i'(\aT) \models X \leftarrow x$ for any $i, X, x$. They are \emph{weakly equivalent}, $\aT \Leftrightarrow \aT'$, if this property only holds when $\mathrm{dom}(i) = \varnothing$.}
\end{definition}
It is immediate that interventions compose: for any interventions $i, i'$, there is an intervention $i' \circ i$ such that $(i' \circ i)(\mathsf{T}) \simeq i'(i(\mathsf{T}))$ for all $\mathsf{T}$. If $i, i'$ are computable then so is $i' \circ i$.


The mere fact that a variable has been intervened on may affect a simulation, even if it has been held fixed to its actual value \citep{IbelingIcard}.
This behavior differs starkly from that of SEMs, so to compare the models we must consider only the class of simulations in which non-counterfactual interventions are idempotent:
\begin{definition}
\label{simfunctional}
\totalpartial{We say $\aT$ is (strongly) \emph{functional} if for any $i$ with $i(\aT) \models v$ and any restriction $i'$ of $v$, $i'(i(\aT))$ and $i(T)$ are weakly equivalent.}{We say $\mathsf{T}$ is \emph{functional} when the following condition holds. Let $i$ be any intervention and $i'$ be such that $i(\aT) \models X \leftarrow i'(X)$ for all $X \in i'$. Then $i'(i(\aT))$ and $i(\aT)$ are weakly equivalent.} 
\end{definition}

We define the causal influence relation $\rightsquigarrow$ in simulations analogously to Defn. \ref{seminfluence}.
We call a simulation \emph{monotone} if it is well-ordered with respect to causal influence: 
\begin{definition}
\label{monotonesimulation}
A simulation $\mathsf{T}$ is \emph{monotone} if there is a computable time map $t : \chi \to \mathbb{N}$ preserving $\rightsquigarrow$.
\end{definition}
It follows from compositionality of interventions that if $\aT$ is monotone under $t$, then so is $i(\aT)$ for any $i$. Here we consider only simulations that have a solution under any computable intervention. Let $\mathcal{T}$ be the class of such simulations that are also functional and monotone.

\subsection{Relation to Existing Work} \label{disc}

Thm. \ref{equivalencetheorem} below shows that the class $\mathcal{M}$ of computable SEMs can essentially be seen as  declarative versions of the class $\mathcal{T}$ of generative programs. Both classes of models naturally accommodate probability. For SEMs, we simply identify a subset of the variables as \emph{exogenous} and associate those variables with an appropriate probability distribution, which in turn induces (conditional and interventional) distributions on the remaining (``endogenous'') variables \citep{Pearl2009}. For simulation programs, we simply add an additional read-only \emph{random bit tape} whose distribution induces random behavior in the machine, which in turn implicitly defines (conditional and interventional) distributions on the variables $\chi$ (cf. \citealt{Freer}). The resulting models encompass a wide range of familiar and powerful formalisms. 

For instance, $\mathcal{M}$ clearly incorporates all computable (Defn. \ref{computablesem}) and recursive (in the sense of \citealt{Pearl2009}) SEMs. Likewise, $\mathcal{T}$ certainly includes Bayesian networks as well as dynamic Bayesian networks \citep{Dean}, in addition to standard dynamic programming algorithms and feedforward neural networks. But crucially, $\mathcal{T}$ also encompasses programs in which variables may depend on infinitely many others, such as those defined in common probabilistic programming languages \citep{BLOG,goodmanetal08,deRaedt,tran2017deep,Pyro}. This includes, e.g., the BLOG program formalizing the radar blip example, among many others.

As in previous work (e.g., \citealt{GallesPearl,Halpern1998,Halpern2000,Zhang}, etc.), the axiomatizations below in \S\ref{condlog} factor out probabilities (in effect, holding all probabilistic choices fixed) and focus only on the deterministic core, as defined above in \S\ref{sec:sems}, \S\ref{sec:simulations}.

\subsection{The Equivalence of Monotone Simulations and Computable SEMs}

\totalpartial{Analogously to Defn. \ref{simequivalence}, we can assess equivalence between between any two causal models, whether simulations or SEMs. It turns out that the classes $\T$ and $\M$ are equivalent in the following sense. 
}{It turns out that $\T$ and $\M$ are equivalent in the following strong sense. Given $M \in \M, \aT \in \T$, write $M \simeq \aT$ if for all $i, v$ we have $i(M) \models v \leftrightarrow i(\aT) \models v$.}
\begin{theorem}
\label{equivalencetheorem}
\totalpartial{

For any $M \in \mathcal{M}$ there is a $\mathsf{T} \in \mathcal{T}$ such that $M \simeq \mathsf{T}$ and vice versa.
}{For any $M \in \mathcal{M}$ there is a $\mathsf{T} \in \mathcal{T}$ such that $M \simeq \mathsf{T}$ and vice versa. The same holds for $\Mtotal, \Ttotal$.}
\end{theorem}
\begin{proof}
Let $M \in \M$. Consider $\mathsf{T} \in \T$ with this pseudocode: $\mathsf{T}$ consists of an infinite loop calling the following subroutine $\textsc{Calc}(X)$ on every $X \in \chi$. The subroutine $\textsc{Calc}(X)$ returns immediately if the variable tape square for $X$ isn't blank. Otherwise it emulates a computation of $\phi(X, \cdot)$ (Defn. \ref{computablesem}) and writes the result to the square for $X$. The \emph{emulated} input variable tape (which holds the second argument of $\phi$) behaves as follows on access to the square for some variable $Y$. If the square for $Y$ in the \emph{original} output variable tape (which $\mathsf{T}$ writes to) is not blank, then the emulated tape holds the value from the original tape. If it is blank, then the simulation computes $t(Y)$ and $t(X)$; if $t(Y) < t(X)$ then $\textsc{Calc}(Y)$ is recursively called so that the $Y$-square is no longer blank and one can proceed as above. Otherwise, the emulated tape contains some arbitrary value in its $Y$-square.

Note $\aT$ is functional since the square for any $X$ can be written only by a call to $\textsc{Calc}(X)$. Next let us show that $M \simeq \mathsf{T}$, which immediately implies that $\mathsf{T}$ is monotone. Let $i$ be an intervention, and suppose $i(M) \models v$. We show by induction that $i(\aT) \models v$. The base case: a variable $X$ with $t(X) = t_0 = \min_{X'\in\chi} t(X')$. If $X \notin i$ then $v(X) = f_X(v') = \phi(X, v')$ for any input $v'$; thus $\textsc{Calc}(X)$ writes $v(X)$. $X \in i$ is trivial. Now consider an $X$ with $t(X) > t_0$. By induction and construction of the emulator, the value output by $\textsc{Calc}(X)$ is $f_X(v')$ at some $v'$ such that $v'(Y) = v(Y)$ for all $Y$ such that $t(Y) < t(X)$. By Defn. \ref{semdefinition}, $f_X(v') = f_X(v) = v(X)$ so we are done. An analogous induction shows that if $i(\aT) \models v$ then $i(M) \models v$.

Now let $\mathsf{T} \in \mathcal{T}$. Define $M \in \M$ by giving the following pseudocode to compute $\phi(X, v)$: emulate the run of $i'(\aT)$, where $i'$ is the restriction of $v$ to all variables but $X$, until reaching a write to the emulated $X$-square, and halt outputting this value. This is computable in the sense of Defn. \ref{computablesem} since the oracle (Defn. \ref{intervention}) for $i'$ on $X'$ can simply check that $X' \neq X$ and then look on the input variable tape to find the required value $v(X')$. If $i(\aT) \models v$, when computing any $\phi(X, v)$ the $i'$ above is a restriction of $v$, so $i(M) \models v$ by functionality (Defn. \ref{simfunctional}). Consider Lem. \ref{equivalencelemma} below for the converse.



\begin{lemma}
\label{equivalencelemma}
Let $\aT \in \Ttotal$, $i$ be an intervention, and $t \in \mathbb{N}$. Suppose $i'$ is the restriction of $i$ to those $X'$ such that $t(X') < t$, with $i(\aT) \models v$ and $i'(\aT) \models v'$. Then $v(X) = v'(X)$ for all $X \notin i$ with $t(X) \le t$.
\end{lemma}
\begin{proof}
Such an $X$ is written to at some point in executing both $i(\aT), i'(\aT)$ so we have finite restrictions $i_{\centernot{\infty}}, i'_{\centernot{\infty}}$ of $i, i'$ such that $i_{\centernot{\infty}}(\aT), i'_{\centernot{\infty}}(\aT)$ write the same value to $X$ as do $i(\aT), i'(\aT)$ resp. Transforming $i'_{\centernot{\infty}}$ into $i_{\centernot{\infty}}$ stepwise (cf. proof of Prop. \ref{Mlocalmediated}) by adding one more variable $Y$ to the domain in each step, if $v(X) \neq v'(X)$ then some step yields a $Y$ such that $\aT \models Y \rightsquigarrow X$ but $t(Y) \ge t(X)$.
\end{proof}
Suppose $i(M) \models v$; we show $i(\aT) \models v$ by induction. The base case is again a variable $X$ with $t(X) = t_0$. If $X \notin i$, apply Lem. \ref{equivalencelemma} choosing $i', t_0$ as $i, t$; the restriction is empty, so the value output for $X$ by $i'(\aT)$ agrees with that output by $\aT$, which in turn agrees with that output by $i(\aT)$ since $t(X) = t_0$ and $X \notin i$. For the inductive case with $t(X) = t > t_0$, choosing $i', t$ in Lem. \ref{equivalencelemma} and applying functionality and composition gives that $i(\aT)$ writes $v(X)$ to its $X$-square.
\end{proof}

\section{The Conditional Logic of Open-Universe Causal Models} \label{condlog}
Thm. \ref{equivalencetheorem} establishes a strong equivalence between simulation models and (computable) SEMs, showing that they encode exactly the same underlying conditional theories. In this section we offer an axiomatic characterization of these conditional theories.

We introduce two formal languages of subjunctive conditionals interpreted in either $\M$ or $\T$ under their respective semantics of intervention. We then find sound and complete axiomatizations of their validities. For reasons of simplicity and elegance, but not necessity, we assume a binary alphabet $\Sigma = \{0, 1\}$ in this section.

\subsection{Syntax and Semantics}
We admit the variables of $\chi$ as atoms and
let $\L_{\mathrm{prop}}$ be the language of propositional formulas over these atoms.
Let $\L_{\mathrm{int}} \subset \L_{\mathrm{prop}}$
be the language of
purely conjunctive, ordered formulas of literals formed from distinct atoms.
To be precise, let us call $\ell$ an $X$-literal if $\ell \in \{X, \lnot X\}$.
Then $\L_{\mathrm{int}}$ consists of
formulas of the form $\ell_{1} \land \dots \land \ell_{n}$, where each $\ell_i$ is an $X_i$-literal for some distinct
$X_1, \dots, X_n \in \chi$.
Each $\alpha \in \L_{\mathrm{int}}$ specifies a finite intervention
by giving fixed values for a fixed list of variables, and we also write $\alpha$ for this intervention.
$\L_{\mathrm{int}}$ includes the empty intervention $\top$ with $\mathrm{dom}(\top) = \varnothing$.

Let $\L_{\mathrm{cond}}$ be the language of formulas
of the form $[\alpha] \beta$ for $\alpha \in \L_{\mathrm{int}}, \beta \in \L_{\mathrm{prop}}$. We call such a formula a
\emph{subjunctive conditional}, and call $\alpha$
the \emph{antecedent} and $\beta$ the \emph{consequent}.
The overall conditional language $\L$ is the language of propositional
formulas over atoms in $\L_{\mathrm{cond}}$.
For $\alpha, \beta \in \L$,
$\alpha \rightarrow \beta$ abbreviates
$\lnot \alpha \lor \beta$,
and $\alpha \leftrightarrow \beta$ means
$(\alpha \rightarrow \beta) \land (\beta \rightarrow \alpha)$.


\totalpartial{To give the semantics of
$\L$, we first define a satisfaction relation between solutions
$v : \chi \to \Sigma$ and formulas of $\L_{\mathrm{prop}}$. For $X \in \chi$,
write $v \models X$ iff $v(X) = 1$. For arbitrary
$\varphi \in \L_{\mathrm{prop}}$, $v \models \varphi$
is defined familiarly by recursion.
Now we may define satisfaction for $\L_{\mathrm{cond}}$. Let $M$ be a model in either $\M$ or $\T$.
Write $M \models [\alpha]\beta$ if
$\alpha(M)$ has
solution $v$ and $v \models \beta$. 
Finally, satisfaction for the entire language $\L$ is defined by recursion.}{We give the semantics of $\L$ as follows. Letting $M$ be either an SEM or a monotone simulation, write $M \models [ \alpha ] X$ if $\alpha(M) \models X \leftarrow 1$ and $M \models [ \alpha ] \lnot X$ if $\alpha(M) \models X \leftarrow 0$}

We also consider an augmented language $\L^+$, defined by extending $\L$ to include additional atoms $X \rightsquigarrow Y$ for $X, Y \in \chi$. This is interpreted as causal influence: $X \rightsquigarrow Y$ holds in model $M$ if $M \models X \rightsquigarrow Y$ (recall Defn. \ref{seminfluence}).

Thm. \ref{equivalencetheorem}
implies that any $\varphi \in \L^+$
is satisfiable in $\T$ iff satisfiable in $\M$.
Hence any semantic property
is oblivious to the class in which to interpret it
and we omit this.

A first
observation is that our language and interpretations do not enjoy the property of \emph{compactness} (the property that every unsatisfiable set of formulas has a finite unsatisfiable subset),
and our axiomatizations are hence only complete relative
to finite sets of assumptions:
\begin{proposition}
The languages $\L, \L^+$ are not compact.
\end{proposition}
\begin{proof}
Enumerate the variables as $\chi = \{ X_i \}_{i \in \mathbb{N}}$. Consider the set $\Omega = \{ [ X_{i+1} ] X_{i} \land [ \lnot X_{i+1} ] \lnot X_{i} : i \in \mathbb{N} \}$. Every finite subset of $\Omega$ is easily seen to be satisfiable. But each formula implies that $X_{i+1} \rightsquigarrow X_i$, contradicting that $t$ is well-founded; thus $\Omega$ is unsatisfiable.
\end{proof}

\subsection{Axiomatizations} \label{axn}

We now give an
axiomatization $\textsf{AX}$ of the validities of $\L$. 
Start with the base system $\textsf{AX}^{\downarrow}_{\mathrm{det}}$ from \cite{IbelingIcard},
obtained by the axioms and rules:
\begin{eqnarray*}
 \textsf{PC}. && \mbox{Propositional calculus} \\
 \textsf{RW}.& & \mbox{From }\beta \rightarrow \beta'\mbox{ infer }[\alpha]\beta \rightarrow [\alpha]\beta' \\
  \textsf{R}. && [\alpha]\alpha \\
 \textsf{K}. && [\alpha](\beta \rightarrow \gamma) \rightarrow ([\alpha]\beta \rightarrow [\alpha]\gamma) \\
 \textsf{F/D}. && [ \alpha ] \lnot \beta \leftrightarrow \lnot [ \alpha ] \beta 
\end{eqnarray*}
The causal axiom of \emph{effectiveness} \citep{GallesPearl,Halpern2000} is encoded by \textsf{RW} and \textsf{R}, while \textsf{F/D} expresses the fact that solutions always exist uniquely. System $\textsf{AX}^{\downarrow}_{\mathrm{det}}$, without \textsf{F/D}---a system significantly weaker than any existing logic of causal conditionals in the literature---has been shown sound and complete with respect to the general class of simulation models, while \textsf{F/D} can be added to axiomatize those programs that are deterministic and/or always halt \citep{IbelingIcard}. In the present setting we are assuming determinism (but allowing a natural incorporation of probability; recall \S\ref{disc}).
 
In order to capture the logic of our restricted classes of SEMs and simulation programs, we add two further axioms to obtain the system $\textsf{AX} = \textsf{AX}^{\downarrow}_{\mathrm{det}} + \textsf{C} + \textsf{Rec}$. 
Axiom $\textsf{C}$ is known variously as
\emph{cautious monotonicity} \citep{Kraus1990} or \emph{composition} \citep{GallesPearl,Pearl2009}, and schema $\textsf{Rec}$
is known as \emph{recursiveness} \citep{Halpern1998,Halpern2000}.
\begin{eqnarray*}
  \textsf{C}. && \left([\alpha]\beta \wedge [\alpha]\gamma\right) \rightarrow [\alpha \wedge \beta]\gamma \\
  \textsf{Rec}. && \big([ \alpha_1 \land X_1 ] \ell_2 \land [ \alpha_1 \wedge \lnot X_1 ] \lnot \ell_2 \land \dots \land \\
  & \quad &
   [ \alpha_{k-1} \land X_{k-1} ] \ell_k \land [ \alpha_{k-1} \wedge \lnot X_{k-1} ] \lnot \ell_k\big) \\
  & \quad & \rightarrow
  \lnot\left(
  [ \alpha_{k} \land X_{k} ] \ell_1 \land [ \alpha_{k} \wedge \lnot X_{k} ] \lnot \ell_1
  \right)
\end{eqnarray*}
In $\textsf{Rec}$, each $\ell_i$ is an $X_i$-literal and $X_1 \neq X_k$. Note that the axiom commonly known as \emph{reversibility} is easily derivable in \textsf{AX} \citep{GallesPearl}.

For $\L^+$, consider the additional axioms
\begin{eqnarray*}
  \textsf{Wit}. \hspace{-.2in} &&
  \left([ \alpha \land X_1 ] \ell_2 \land [ \alpha \wedge \lnot X_1 ] \lnot \ell_2\right)
   \rightarrow X_1 \rightsquigarrow X_2 \\
  \textsf{Rec}^+. \hspace{-.2in} && \left(X_1 \rightsquigarrow X_2 \land \dots \land
   X_{k-1} \rightsquigarrow X_k\right)
 \rightarrow X_k \nrightsquigarrow X_1 \\
  \textsf{Trans}. \hspace{-.2in}  && \left(X \rightsquigarrow Y \land Y \rightsquigarrow Z\right) \rightarrow X \rightsquigarrow Z
\end{eqnarray*}
where in $\textsf{Wit}$, $\ell_2$ is an $X_2$-literal and $X_1 \neq X_2$. Let $\textsf{AX}^+ = \textsf{AX} + \textsf{Wit} + \textsf{Rec}^+$.

We are now in a position to present our three axiomatization results. 

\begin{theorem}
\label{complete:ax}
  $\textsf{AX}$ is sound and complete with respect to the validities of $\L$.
\end{theorem}
\begin{proof}
\totalpartial{
By Thm. \ref{equivalencetheorem} it suffices to carry out the proof for $\M$ only. Soundness is straightforward. As for completeness, we show that any $\varphi \in \L$ consistent with $\textsf{AX}$ is satisfiable. Let $\chi_{\varphi} \subset \chi$ be the set of variables that appear as atoms in $\varphi$, let $\L_{\varphi} \subset \L$ be the fragment of $\L$ using only atoms from $\chi_{\varphi}$, and extend $\varphi$ to a maximal consistent set $\Gamma \subset \L_{\varphi}$. We construct an SEM $M$ satisfying all of $\Gamma$, and, in particular, $\varphi$.


Let us first give $t$.
Define an irreflexive relation $\prec$ on $\chi_{\varphi}$ by $X_1 \prec X_2$ if any instance of the schema $[\alpha_1 \land X_1 ] \ell_2 \land [ \alpha_1 \land \lnot X_1 ] \lnot \ell_2$ is in $\Gamma$,
where $\ell_2$ is an $X_2$-literal. By consistency with $\textsf{Rec}$ there is a total order consistent with
$\prec$. 
Assign the least element to $t = 0$, and iteratively remove least elements to obtain a $t$ injective on $\chi_{\varphi}$.

Now we may give the structural equations. Note that for every $X \in \chi_{\varphi}$ and $\alpha \in \L_{\mathrm{int}} \cap \L_{\varphi}$ there is exactly one formula of the form $[ \alpha ] \ell$, where $\ell$ is an $X$-literal, in $\Gamma$: the forward direction of $\textsf{F/D}$ shows there is at most one, and the backward direction shows there is one. Thus form a vector $v_{\alpha} : \chi \to \{0, 1\}$ by $v_{\alpha}(X) = 0, 1$ if $[ \alpha ] \lnot X, [ \alpha ] X$ are in $\Gamma$ respectively, and $v_{\alpha}(X) = 0$ if $X \notin \chi_{\varphi}$. Suppose $\alpha$ includes exactly the variables $X' \in \chi_{\varphi}$ for which $t(X') < t(X)$. Let $\alpha, \dots$ represent an extension of $\alpha$ to all of $\chi$; then define $f_X$ on any such extension as $f_X(\alpha, \dots) = v_{\alpha}(X)$. For all $X \notin \chi_{\varphi}$, let $f_X(\cdot) = 0$ and $t(X) = 0$.

It remains to prove that $M$ satisfies every formula of $\Gamma$, for which it suffices to show that $M \models [ \alpha ] X$ iff $v_{\alpha}(X) = 1$. To this effect we prove that $v(X) = v_{\alpha}(X)$, where $\alpha(M) \models v$, by induction on the time $t(X)$.
If $X \in \alpha$ then by $\textsf{RW}, \textsf{R}$ we have $v(X) = v_{\alpha}(X) = \alpha(X)$; thus suppose $X \notin \alpha$.
Consider an $\alpha'$ that includes the value $v(X') = v_{\alpha}(X')$, where the equality holds inductively, for each $X' \in \chi_{\varphi}$ such that $t(X') < t(X)$. Thus by $\textsf{K}$ we have that $[ \alpha ] \alpha' \in \Gamma$.
Suppose that $v, v_{\alpha}$ differ at $X$, e.g. $v(X) = 0$ but $v_{\alpha}(X) = 1$. Then
$[ \alpha ] X \in \Gamma$ so by $\textsf{C}$,
$[ \alpha \land \alpha' ] X \in \Gamma$. By $\textsf{C}$ and $\textsf{Rec}$ we must have $[ \alpha' ] X \in \Gamma$ also. But this term specifies $f_X$, and hence this contradicts that $v(X) = 0$.}{}
\end{proof}



\begin{theorem}
  \label{complete:ax+}
$\textsf{AX}^+$ is sound and complete with respect to the validities of $\L^+$.
\end{theorem}
\begin{proof}
Soundness is again straightforward; we show a consistent $\varphi \in \L^+$ is satisfiable. Extend to a maximal consistent set $\Gamma$ using only $\chi_{\varphi}$-atoms. Define a relation $\prec^+$ by $X \prec^+ Y$ if $X \rightsquigarrow Y \in \Gamma$. By consistency with $\textsf{Rec}^+$ we have that there is a total order consistent with $\prec^+$. Now, reproduce the construction from the proof of Thm. \ref{complete:ax} on $\Gamma$. By consistency with $\textsf{Wit}$, $\prec$ is consistent with $\prec^+$, so we may extend them to the same total order, from which we obtain $t$. Let $M$ be the model thus constructed.

Now, obtain a new model $M'$ by extending $M$ as follows. For every pair $X, Y$, introduce a new variable $Z_{X,Y}$ with $f_{Z_{X,Y}}(\cdot) = 0$ and $t(Z_{X,Y}) = 0$. It is possible to add such variables since $\chi$ is infinite. Now, modify the structural functions for each $Y \in \chi_{\varphi}$ as follows. Labelling the variables $X'$ for which $X' \prec^+ Y$ as $X_1, \dots, X_k$:
\[
  f_Y(v) = \begin{cases}
  f_Y(v), & v(Z_{X', Y'}) = 0 \text{ for all } X', Y' \\
  X_1, &  v(Z_{X_1, Y}) = 1 \text{ and } v(Z_{X', Y'}) = 0\\ & \text{whenever } (X', Y') \neq (X_1, Y)\\
  & \dots \\
  X_k, &  v(Z_{X_k, Y}) = 1 \text{ and } v(Z_{X', Y'}) = 0\\ & \text{whenever } (X', Y') \neq (X_k, Y)\\
  0, & \text{otherwise} 
  \end{cases}
\]
This modification is admissible because $\prec^+$ is consistent with $t$. We claim that $M'$ satisfies all of $\Gamma$. Any $\psi \in \Gamma \cap \L$ is clearly satisfied: $\psi$ is satisfied by $M$ and in $M'$ we have $v(Z_{X, Y}) = 0$ for all $X, Y$, so the structural equations are unmodified from those of $M$. Thus, suppose that an extended atom $X \rightsquigarrow Y \in \Gamma$. Consider an intervention $\alpha$ that holds $Z_{X, Y}$ to $1$, and holds all other $Z_{X', Y'}$ to $0$. Then under $\alpha$, we have the modified structural equation $f_Y(v) = v(Y) = v(X)$, so $M' \models X \rightsquigarrow Y$.

Now suppose $X \nrightsquigarrow Y \in \Gamma$. Suppose toward a contradiction that $M' \models X \rightsquigarrow Y$. Then there is an intervention $\alpha$ under which toggling $X$ changes $Y$. If $\alpha$ does not set any $Z_{X', Y'}$-variables, then the structural equations are the same in $M, M'$ and we have that $M \models [ \alpha \land X ] \ell_Y \land [ \alpha \land \lnot X ] \lnot \ell_Y$, where $\ell_Y$ is a $Y$-literal, contradicting consistency with $\textsf{Wit}$. If $\alpha$ sets any $Z_{X', Y'}$ to $1$ for $Y' \neq Y$, then the value of $Y$ is fixed to $0$. The only case remaining is when only $\alpha(Z_{X', Y}) = 1$ for some $X'$. If $X' = X$ then there is a contradiction since the modification was only made if $X \rightsquigarrow Y \in \Gamma$. Thus $X' \neq X$. The structural equation for $Y$ then becomes $f_Y(v) = v(Y) = v(X')$; further, $X'$ is fixed to $0$ since $v(Z_{X', Y}) = 1$, so this is impossible.
\end{proof}


\begin{theorem}
  \label{complete:ax+trans}
$\textsf{AX}^+ + \textsf{Trans}$ is sound and complete with respect to the validities of $\L^+$ over $\Mlocal$.
\end{theorem}
\begin{proof}
Soundness follows since causal influence in $\Mlocal$ is transitive, as we remarked before. To show completeness, as before, extend a given consistent $\varphi \in \L^+$ to a maximal consistent set $\Gamma$, and define $\prec^+$ by $X \prec^+ Y$ if $X \rightsquigarrow Y \in \Gamma$. By consistency with $\textsf{Rec}^+, \textsf{Trans}$, $\prec^+$ is a strict partial order. Again reproduce the construction from the proof of Thm. \ref{complete:ax}; by consistency with $\textsf{Wit}$, $\prec$ is consistent with $\prec^+$ and we extend them to the same total order, obtaining $t$ and $M$.

First convert $M$ to a model $M' \in \Mlocal$. In $M$, the structural function for a variable $Y$ may be written as $f_Y( X_1, \dots, X_k )$, where $X_1, \dots, X_k \in \chi_{\varphi}$ are precisely the variables preceding $Y$. To form $M'$, for each $X_i$ that does not satisfy $t(X_i) = t(Y) - 1$, add a chain of variables $W'_{X_i, t(X_i) + 1}, \dots, W'_{X_i, t(Y)-1}$ such that $t(W'_{X_i, t}) = t$ and with structural equations $v(W'_{X_i, t(X_i) + 1}) = v(X_i), \quad v(W'_{X_i, t}) = v(W'_{X_i, t-1})$ for all $t(X_i) + 1 < t < t(Y)$, and change the structural function for $Y$ to $f_Y(W'_{X_i, t(Y)-1}, \dots, W'_{X_i, t(Y)-1} , \dots, W'_{X_k, t(Y)-1})$. It is clear that $M' \models \psi$ iff $M \models \psi$ for any $\psi \in \L^+$.

Now modify $M'$ to obtain a model $M'' \in \Mlocal$ that we will show to satisfy all of $\Gamma$. To form $M''$, for every pair $X, Y$ such that $X \prec^+ Y$, add a chain $W_{X, Y, t(X) + 1}, \dots, W_{X, Y, t(Y) - 1}$ and a single variable $Z_{X, Y}$.
These have the times $t(W_{X, Y, t'}) = t'$, and $t(Z_{X, Y}) = t(Y) - 1$; and the structural equations $v(W_{X, Y, t(X) + 1}) = v(X)$, $v(W_{X, Y, t'}) = v(W_{X, Y, t'-1})$ for $t(X) < t' < t(Y)$, and $v(Z_{X, Y}) = 0$.
Modify the structural equation for $Y$ as follows,
labelling the variables $X'$ for which $X' \prec^+ Y$ as $X_1, \dots, X_k$:
\[
  f_Y(v) = \begin{cases}
  f_Y(v), & \hspace{-0.1cm} v(Z_{X', Y}) = 0 \text{ for all } X' \\
  v(W_{X_1, Y, t(Y)-1}), & v(Z_{X_1, Y}) = 1 \text{ and }\\
  & \hspace{-0.5cm} v(Z_{X', Y}) = 0, \text{ if } X' \neq X_1 \\
  & \dots \\
  v(W_{X_k, Y, t(Y)-1}), & v(Z_{X_k, Y}) = 1 \text{ and }\\
  & \hspace{-0.5cm} v(Z_{X', Y}) = 0, \text{ if } X' \neq X_k \\
  0,& \text{otherwise}
  \end{cases}
\]
The modification is admissible because all variables involved are assigned to $t(Y) - 1$. We claim that $M''$ satisfies all of $\Gamma$. Again, this holds for any $\L$-literal, so suppose $X \rightsquigarrow Y \in \Gamma$. Under an intervention that holds the $Z_{X, Y}$ to $1$, we have that $v(Y) = v(W_{X, Y, t(Y)-1}) = \dots = v(W_{X, Y, t(X) + 1}) = v(X)$ so that $M'' \models X \rightsquigarrow Y$.

Now suppose that $M'' \models X \rightsquigarrow Y$. The intervention witnessing this fixes either none, or one of the $Z_{X_i, Y}$ to $1$. If it fixes none, then taking the values of all the $\chi_{\varphi}$ variables into an intervention gives a witness that $M \models X \rightsquigarrow Y$, so by $\textsf{Wit}$, $X \rightsquigarrow Y \in \Gamma$. Thus suppose it fixes just one $Z_{X_i, Y}$ to $1$. This gives that $X_i \rightsquigarrow Y \in \Gamma$, with $t(X) \le t(X_i) < t(Y)$. Now $M'' \models X \rightsquigarrow X_i$; we thus repeat the argument obtaining a sequence $X \prec^+ X'_1 \prec^+ \dots \prec^+ X'_k \prec^+ Y$. Since $\prec^+$ is transitive (from $\textsf{Trans}$) we finally obtain $X \rightsquigarrow Y \in \Gamma$.
\end{proof}


\subsection{Complexity}

We turn to the complexity of deciding satisfiability for an arbitrary formula $\varphi \in \L$.
Let $\mathcal{L}^-_{\mathrm{int}(\varphi)} \subset \mathcal{L}_{\mathrm{int}} \cap \mathcal{L}_{\varphi}$ be the subset of interventions appearing in $\varphi$.
For each $X \in \chi_\varphi$, let $L(X) = \{ X, \lnot X\}$ be the set of $X$-literals.
Then define a subset $\Delta_\varphi \subset \mathcal{L}_{\mathrm{\varphi}}$ by $\Delta_\varphi = \Big\{ \bigwedge_{\substack{\alpha \in \mathcal{L}^-_{\mathrm{int}(\varphi)}}} \big([\alpha] \bigwedge_{X \in \chi_\varphi} \ell_X^\alpha \big) : \ell_X^\alpha \in L(X) \text{ for each } \alpha \in \mathcal{L}^-_{\mathrm{int}(\varphi)}, X \in \chi_\varphi \Big\}$; we have $| \Delta_\varphi| = 2^{|\chi_\varphi| |\mathcal{L}^-_{\mathrm{int}(\varphi)}|}$.
Lem.~\ref{lem:complexity:1} and \ref{lem:complexity:2} below are straightforward; cf. \cite[Lem.~13]{ibelingicard2020} and \cite[App.~A]{BCII2020}.
\begin{lemma}
\label{lem:complexity:1}
 $\models \varphi \leftrightarrow \bigvee_{\substack{\delta \in \Delta_\varphi\\ \models \delta \rightarrow \varphi}} \delta$.
 \qed
\end{lemma}
Where $\prec$ is a total order of $\chi_\varphi$, assign labels $\chi_\varphi = \{X_1, \dots, X_n\}$ so that $X_1 \prec \dots \prec X_n$
and let $\mathcal{M}_{\prec}$ be the class of models in which $t(X_i) = i-1$ for all $X_i \in \chi_\varphi$, and $t(X) = n$ for all $X \notin \chi_\varphi$.
\begin{lemma}
\label{lem:complexity:2}
Suppose $\alpha \in \mathcal{L}_{\mathrm{int}} \cap \mathcal{L}_{\varphi}$. Let $\ell_{X} \in L(X)$ be a choice of $X$-literal for each $X \in \chi_\varphi$ such that $\ell_X = \alpha(X)$ whenever $X \in \alpha$.\footnote{Here $\alpha$ represents an intervention thought of as a partial function; see Def.~\ref{semintervention}.}
Then $M \models [\alpha] \bigwedge_{X \in \chi_\varphi} \ell_X \leftrightarrow \bigwedge_{\substack{0 \le m < n \\ X_{m+1} \notin \alpha}} [\ell_{X_1} \land \dots \land \ell_{X_m}] \ell_{X_{m+1}}$ for every $M \in \mathcal{M}_\prec$.
\qed
\end{lemma}
\begin{proposition}
\label{prop:satl:np}
 Satisfiability for $\L$ is $\mathsf{NP}$-complete.
\end{proposition}
\begin{proof}
 $\mathsf{NP}$-hardness is trivial. The proof of Thm.~\ref{complete:ax} shows that $\varphi$ is satisfiable iff satisfiable in $\mathcal{M}_\prec$ for some $\prec$.
 In any $\prec$ and for any $M \in \mathcal{M}_\prec$, Lem.~\ref{lem:complexity:1} and \ref{lem:complexity:2} imply that $M \models \varphi$ iff $M \models \bigvee_{\substack{\bigwedge_\alpha\big( [\alpha] \bigwedge_{X} \ell_X^\alpha \big) \models \varphi \\ \ell_X^\alpha = \alpha(X) \text{ for all }\alpha, X \in \alpha}} \bigwedge_{\substack{\alpha \in \mathcal{L}^-_{\mathrm{int}(\varphi)} \\0 \le m < n \\ X_{m+1} \notin \alpha}} [\ell^\alpha_{X_1} \land \dots \land \ell^\alpha_{X_m}] \ell^\alpha_{X_{m+1}}$.
Every inner conjunction here is of polynomial length, and is satisfiable iff $\ell^{\alpha}_{X_{m+1}} = \ell^{\alpha'}_{X_{m+1}}$ whenever $\ell^\alpha_{X_1} \land \dots \land \ell^\alpha_{X_m} = \ell^{\alpha'}_{X_1} \land \dots \land \ell^{\alpha'}_{X_m}$.
Thus a specification of order $\prec$ and values $\big\{\ell^\alpha_{X_{m+1}}\big\}_{\substack{\alpha, m\\ X_{m+1} \notin \alpha}}$ furnishes a polynomial certificate, with the check above accomplishable in $\mathsf{P}$. For an explicit verifier, see \cite[Lem.~16]{ibelingicard2020}.
\end{proof}

\subsection{Relation to Existing Work}

Thm. \ref{complete:ax} shows that the class of open-universe generative programs defined in \S\ref{sec:simulations} satisfies a natural set of axioms. In fact, the system \textsf{AX} encompasses all of the principles about counterfactual conditionals used in the complete identification algorithm of \cite{Shpitser}. This, together with Thm. \ref{equivalencetheorem}, lends at least some credence to the idea that these ``merely implicit'' causal models can be understood on a par with more familiar explicit causal representations such as SEMs. 

The previous literature on causal conditionals has explored classes of axioms and models that are quite different from those considered here. For instance, in addition to studying the (finite) recursive structural equation models, \cite{Halpern1998,Halpern2000} axiomatizes the class of SEMs that have a unique solution (but may not be recursive), as well as the class of all SEMs built of arbitrary equations, which may in general lack solutions. Similarly, \cite{Zhang} considers classes of SEMs with desirable \emph{sets} of solutions. Some of the central axioms for these classes make reference to all variables of the signature and thus cannot be translated into the open-universe setting. More fundamentally, it is not evident which of these models have an adequate procedural interpretation. 

An analogous generalization in our setting might be to consider simulation models that ``crash'' under certain interventions and fail to have a solution.
If we allow $\phi$ (Defn. \ref{computablesem}) to be only partial computable on $\chi \times F$, the construction from Thm. \ref{equivalencetheorem} shows that every partial SEM has an equivalent simulation model in this wider class; but simple counterexamples show the reverse direction does not hold.
Logically, only the forward direction of axiom $\textsf{F/D}$ remains sound, and we leave the question of axiomatizing this wider class for future work.

The relation $\rightsquigarrow$ has been used as a defined symbol in previous work (e.g., \citealt{Halpern2000}). Because the relevant definition again relies on reference to all variables in the signature, this is not possible in the open-universe setting. Thms. \ref{complete:ax+} and \ref{complete:ax+trans} are thus new in the present work. 

\section{Conclusion}

We have identified two equivalent classes of models---one declarative, one procedural---formalizing the notion of an \emph{open-universe causal model}. Both classes validate an intuitive and familiar set of principles about subjunctive conditionals and the relation of causal influence. This highlights an important class of \emph{implicit} generative models that can plausibly be treated as genuine causal models, on a par with (an infinitary generalization of computable, recursive) structural equation models. More detailed work is of course needed to identify concrete cases in which components of learned generative models support legitimate causal counterfactuals (see, e.g., \citealt{Besserve} for progress on this question). 

From an axiomatic perspective, it would be desirable to extend the present treatment to the full probabilistic setting, since, as remarked in \S\ref{disc}, both classes of models can be augmented with a natural probabilistic source. Axioms for probabilistic formal systems are well studied (e.g., \citealt{Fagin}). In the direction of a fully formalized do-calculus \citep{Pearl2009}, one would like to embed an axiom system like \textsf{AX} into an appropriate probability calculus, and combine these with a logic of \emph{direct} causal influence (a direct and probabilistic version of $\rightsquigarrow$), so that the do-calculus rules could be expressed and studied in a precise formal system. While the basic definitions would be clear for SEMs (and this could of course already be investigated for finite SEMs), the extension to simulation programs is less clear. Existing identifiability algorithms require specific assumptions about exogenous noise variables, e.g., that each has only two endogenous children \citep{Shpitser}. Some work would need to done to ensure that probabilistic programs (or Turing machines) satisfy analogous restrictions.  

Finally, other extensions to the languages considered here would also be natural to investigate. We studied one higher-order relation, namely $\rightsquigarrow$, which involves quantification over an infinite domain (the space of interventions). One of the advantages of open-universe models is precisely that they enable reasoning beyond the propositional level. Thus, e.g., systems for reasoning about causal and counterfactual statements involving explicit quantification (as in the examples from \S\ref{intro}) are easily motivated, and ought to be understood.


\bibliographystyle{apalike}
\bibliography{uai2019}

\end{document}